\documentclass[twoside,11pt]{article}

\usepackage{jmlr2e}
\usepackage{amsmath}
\usepackage{algorithm}
\usepackage{algorithmic}
\usepackage{bbm}

\newcommand{\mb}{\mathbf}
\newcommand{\mc}{\mathcal}
\newcommand{\our}{{Gadam}}

\ShortHeadings{IFM LAB TUTORIAL SERIES \# 1, COPYRIGHT \copyright IFM LAB}{JIAWEI ZHANG, IFM LAB DIRECTOR}
\firstpageno{1}

\begin{document}

\title{Gradient Descent based Optimization Algorithms for Deep Learning Models Training}

\author{\name Jiawei Zhang \email jiawei@ifmlab.org \\
      \addr {Founder and Director}\\
      {Information Fusion and Mining Laboratory}\\
       (First Version: February 2019; Revision: March 2019.)}

\maketitle

\begin{abstract}

In this paper, we aim at providing an introduction to the gradient descent based optimization algorithms for learning deep neural network models. Deep learning models involving multiple nonlinear projection layers are very challenging to train. Nowadays, most of the deep learning model training still relies on the back propagation algorithm actually. In back propagation, the model variables will be updated iteratively until convergence with gradient descent based optimization algorithms. Besides the conventional vanilla gradient descent algorithm, many gradient descent variants have also been proposed in recent years to improve the learning performance, including Momentum, Adagrad, Adam, Gadam, etc., which will all be introduced in this paper respectively.

\end{abstract}


\tableofcontents

\section{Introduction}\label{ref:sec_intro}

In the real-world research and application works about deep learning, training the deep models effectively still remains one of the most challenging work for both researchers and practitioners. By this context so far, most of the deep model training is still based on the back propagation algorithm, which propagates the errors from the output layer backward and updates the variables layer by layer with the gradient descent based optimization algorithms. Gradient descent plays an important role in training the deep learning models, and lots of new variant algorithms have been proposed in recent years to further improve its performance. Compared with the high-order derivative based algorithms, e.g., Newton's method, etc., gradient descent together with its various variant algorithms based on first-order derivative are much more efficient and practical for deep models. In this paper, we will provide a comprehensive introduction to the deep learning model training algorithms.

Formally, given the training data $\mc{T} =\{(\mb{x}_1, \mb{y}_1), (\mb{x}_2, \mb{y}_2), \cdots, (\mb{x}_n, \mb{y}_n)\}$ involving $n$ pairs of feature-label vectors, where the feature vector $\mb{x}_i \in \mathbb{R}^{d_x}$ and label vector $\mb{y}_i \in \mathbb{R}^{d_y}$ are of dimensions $d_x$ and $d_y$ respectively. We can represent the deep learning model used to classify the data as a mapping: $F(\cdot; \boldsymbol{\theta}): \mc{X} \to \mc{Y}$, where $\mc{X}$ and $\mc{Y}$ denotes the feature and label space respectively and $\boldsymbol{\theta}$ denotes the variable vector involved in the mapping. Formally, given an input data instance featured by vector $\mb{x}_i \in \mc{X}$, we can denotes the output by the deep learning model as $\hat{\mb{y}}_i=F(\mb{x}_i; \boldsymbol{\theta})$. Compared against the true label vector $\mb{y}_i$ of the input instance, we can represent the introduced error by the model with the loss term $\ell(\hat{\mb{y}}_i, \mb{y}_i)$. In the remaining part of this paper, we may use error and loss interchangeably without any differentiations.

Many different loss functions, i.e., $\ell(\cdot, \cdot)$, have been proposed to measure the introduced model errors. Some representative examples include
\begin{itemize}

\item \textit{Mean Square Error (MSE)}: 
\begin{equation}
\ell_{MSE}(\hat{\mb{y}}_i, \mb{y}_i) = \frac{1}{d_y} \sum_{j=1}^{d_y} \left(\mb{y}_i(j) - \hat{\mb{y}}_i(j) \right)^2.
\end{equation}

\item \textit{Mean Absolute Error (MAE)}:
\begin{equation}
\ell_{MAE}(\hat{\mb{y}}_i, \mb{y}_i) = \frac{1}{d_y} \sum_{j=1}^{d_y} \left| \mb{y}_i(j) - \hat{\mb{y}}_i(j) \right|.
\end{equation}

\item \textit{Hinge Loss}:
\begin{equation}
\ell_{Hinge}(\hat{\mb{y}}_i, \mb{y}_i) = \sum_{j \neq l_{\mb{y}_i}} \max(0, \hat{\mb{y}}_i(j) - \hat{\mb{y}}_i(l_{\mb{y}_i}) + 1),
\end{equation}
where $l_{\mb{y}_i}$ denotes the true class label index for the instance according to vector $\mb{y}_i$.

\item \textit{Cross Entropy Loss}:
\begin{equation}
\ell_{CE}(\hat{\mb{y}}_i, \mb{y}_i) = - \sum_{j=1}^{d_y} \mb{y}_i(j) \log \hat{\mb{y}}_i(j).
\end{equation}

\end{itemize}

Depending on the number of classes involved as well as the activation functions being used for the model, the \textit{cross entropy} function can be further specified into the \textit{sigmoid cross entropy} (for the binary classification tasks with \textit{sigmiod} function as the activation function for the output layer) and \textit{softmax cross entropy} (for the multi-class classification tasks with \textit{softmax} function as the activation function for the output layer). 

The \textit{mean square error} and \textit{mean absolute error} are usually used for the regression tasks, while the \textit{hinge loss} and \textit{cross entropy} are usually used for the classification tasks instead. Besides these loss functions introduced above, many other loss functions can be used for the specific learning problems as well, e.g., \textit{huber loss}, \textit{cosine distance}, and the readers may refer to the referred articles for more information.

To provide more information about the \textit{hinge loss} and the \textit{cross entropy loss} functions, we will use an example to illustrate their usages in real-world problems.

\begin{example}

\begin{figure}[t]
    \centering
    \includegraphics[width=0.8\textwidth]{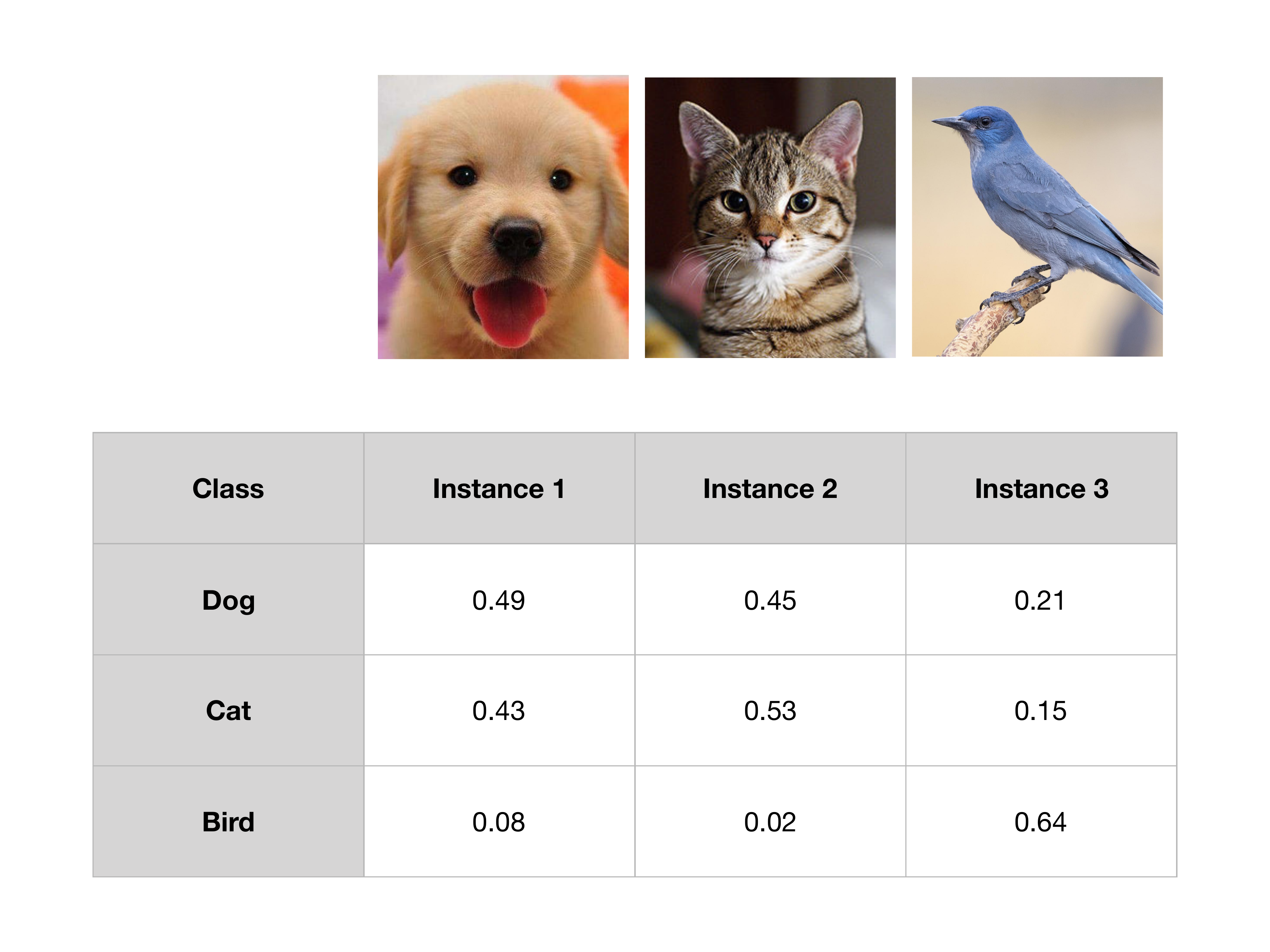}
    \caption{An Example to Illustrate the Hinge Loss and Cross Entropy Loss Functions.}
    \label{fig:chap_gradient_descent_sec1_loss_function_example}
\end{figure}

As show in Figure~\ref{fig:chap_gradient_descent_sec1_loss_function_example}, based on the three input images (i.e., the data instances), we can achieve their prediction labels by a deep learning model as shown in the bottom table. For the dog input image (i.e., instance 1), we know its true label should be $\mb{y}_1 = [1, 0, 0]^\top$ and the prediction label is $\hat{\mb{y}}_1 = [0.49, 0.43, 0.08]^\top$, where these three entries in the vector correspond to the \textit{dog}, \textit{cat} and \textit{bird} class labels respectively. Similarly, we can represent the true labels and prediction labels of the cat image and bird image (i.e., instance 2 and instance 3) as vectors $\mb{y}_2 = [0, 1, 0]^\top$, $\hat{\mb{y}}_2 = [0.45, 0.53, 0.02]^\top$, $\mb{y}_3 = [0, 0, 1]^\top$, and $\hat{\mb{y}}_3 = [0.21, 0.15, 0.64]^\top$ respectively.
 
Based on the true labels and the prediction labels, we can represent the introduced loss terms by the deep learning model as follows
\begin{itemize}

\item \textit{Instance 1}: For instance $1$, we know its true label is dog (i.e., entry $1$ in the label vector), and its true label index is $l_{\mb{y}_1} = 1$. Therefore, we can represent the introduced \textit{hinge loss} for instance $1$ as follows:
\begin{align}
\begin{aligned}
\ell_{Hinge}(\hat{\mb{y}}_1, \mb{y}_1) &= \sum_{j \neq 1} \max(0, \hat{\mb{y}}_1(j) - \hat{\mb{y}}_1(1) + 1)\\
&= \max(0, \hat{\mb{y}}_1(2) - \hat{\mb{y}}_1(1) + 1) +  \max(0, \hat{\mb{y}}_1(3) - \hat{\mb{y}}_1(1) + 1)\\
&= \max(0, 0.43 - 0.49 + 1) +  \max(0, 0.08 - 0.49 + 1)\\
&= 1.53.
\end{aligned}
\end{align}

\item \textit{Instance 2}: For instance $2$, we know its true label is cat (i.e., entry $2$ in the label vector), and its true label index will be $l_{\mb{y}_2} = 2$. Therefore, the introduced \textit{hinge loss} for instance $2$ can be represented as
\begin{align}
\begin{aligned}
\ell_{Hinge}(\hat{\mb{y}}_2, \mb{y}_2) &= \sum_{j \neq 2} \max(0, \hat{\mb{y}}_2(j) - \hat{\mb{y}}_2(2) + 1)\\
&= \max(0, \hat{\mb{y}}_2(1) - \hat{\mb{y}}_2(2) + 1) +  \max(0, \hat{\mb{y}}_2(3) - \hat{\mb{y}}_2(2) + 1)\\
&= \max(0, 0.45 - 0.53 + 1) +  \max(0, 0.02 - 0.53 + 1)\\
&= 1.41.
\end{aligned}
\end{align}

\item \textit{Instance 3}: Similarly, for the input instance $3$, we know its true label index is $l_{\mb{y}_3} = 3$, and the introduced \textit{hinge loss} for the instance can be denoted as
\begin{align}
\begin{aligned}
\ell_{Hinge}(\hat{\mb{y}}_3, \mb{y}_3) &= \sum_{j \neq 3} \max(0, \hat{\mb{y}}_3(j) - \hat{\mb{y}}_3(3) + 1)\\
&= \max(0, \hat{\mb{y}}_3(1) - \hat{\mb{y}}_3(3) + 1) +  \max(0, \hat{\mb{y}}_2(2) - \hat{\mb{y}}_3(3) + 1)\\
&= \max(0, 0.21 - 0.64 + 1) +  \max(0, 0.15 - 0.64 + 1)\\
&= 1.08.
\end{aligned}
\end{align}

\end{itemize}

Meanwhile, based on the true label and the prediction label vectors of these $3$ input data instances, we can represent their introduced \textit{cross entropy loss} as follows:
\begin{itemize}

\item \textit{Instance 1}: 
\begin{align}
\begin{aligned}
\ell_{CE}(\hat{\mb{y}}_1, \mb{y}_1) &= - \sum_{j=1}^{d_y} \mb{y}_1(j) \log \hat{\mb{y}}_1(j)\\
&= - 1 \times \log 0.49 - 0 \times \log 0.43 - 0 \times \log 0.08\\
&= 0.713.
\end{aligned}
\end{align}

\item \textit{Instance 2}:
\begin{align}
\begin{aligned}
\ell_{CE}(\hat{\mb{y}}_2, \mb{y}_2) &= - \sum_{j=1}^{d_y} \mb{y}_2(j) \log \hat{\mb{y}}_2(j)\\
&= - 0 \times \log 0.45 - 1 \times \log 0.53 - 0 \times \log 0.02\\
&= 0.635.
\end{aligned}
\end{align}

\item \textit{Instance 3}: 
\begin{align}
\begin{aligned}
\ell_{CE}(\hat{\mb{y}}_3, \mb{y}_3) &= - \sum_{j=1}^{d_y} \mb{y}_3(j) \log \hat{\mb{y}}_3(j)\\
&= - 0 \times \log 0.21 - 0 \times \log 0.15 - 1 \times \log 0.64\\
&= 0.446.
\end{aligned}
\end{align}
\end{itemize}

\end{example}

Furthermore, based on the introduced loss for all the data instances in the training set, we can represent the total introduced loss term by the deep learning model as follows
\begin{equation}
\mc{L}(\boldsymbol{\theta}; \mc{T}) = \sum_{(\mb{x}_i, \mb{y}_i) \in \mc{T}}   \ell (\hat{\mb{y}}_i, \mb{y}_i) = \sum_{(\mb{x}_i, \mb{y}_i) \in \mc{T}}  \ell (F(\mb{x}_i; \boldsymbol{\theta}), \mb{y}_i).
\end{equation}

By fitting the training data, the training of the deep learning model aims at minimizing the loss introduced on the training set so as to achieve the optimal variables, i.e., the optimal $\boldsymbol{\theta}$. Formally, the objective function of deep learning model training can be represented as follows:
\begin{equation}
\min_{\boldsymbol{\theta} \in \mathit{\Theta}} \mc{L}(\boldsymbol{\theta}; \mc{T}),
\end{equation}
where $\mathit{\Theta}$ denotes the variable domain and we can have $\mathit{\Theta} = \mathbb{R}^{d_\theta}$ ($d_\theta$ denotes the dimension of variable vector $\boldsymbol{\theta}$) if there is no other constraints on the variables to be learned. To address the objective function, in the following sections, we will introduce a group of gradient descent based optimization algorithms, which can learn the globally optimal or locally optimal variables for the deep learning models in the case if the loss function is convex or non-convex respectively. A more comprehensive experimental analysis about their performance will be provided at the end of this paper with some hands-on experiments on some real-world datasets.


\section{Conventional Gradient Descent based Learning Algorithms}\label{ref:sec_GD}
In this section, we will introduce the conventional \textit{vanilla gradient descent}, the \textit{stochastic gradient descent} and the \textit{mini-batch gradient descent} algorithms. The main differences among them lie in the amount of training data used to update the model variables in each iteration. They will also serve as the base algorithms for the variant algorithms to be introduced in the following sections.

\subsection{Vanilla Gradient Descent}\label{ref:subsec_GD}

\textit{Vanilla gradient descent} is also well-known as the the \textit{batch gradient descent}. Given the training set $\mathcal{T}$ as introduced in the previous section, the \textit{vanilla gradient descent} optimization algorithm optimizes the model variables iteratively with the following equation:
\begin{equation}
\boldsymbol{\theta}^{(\tau)} = \boldsymbol{\theta}^{{(\tau-1)}} - \eta \cdot \nabla_{\boldsymbol{\theta}} \mc{L}(\boldsymbol{\theta}^{{(\tau-1)}}; \mc{T}),
\end{equation}
where $\tau \ge 1$ denotes the updating iteration.

In the above updating equation, the physical meanings of notations $\boldsymbol{\theta}^{(\tau)} $, $\nabla_{\boldsymbol{\theta}} \mc{L}(\boldsymbol{\theta}^{{(\tau-1)}}; \mc{T})$ and $\eta$ are illustrated as follows:
\begin{enumerate}
\item Term $\boldsymbol{\theta}^{(\tau)} $ denotes the model variable vector updated via iteration $\tau$. At $\tau = 0$, vector $\boldsymbol{\theta}^{(0)}$ denotes the initial model variable vector, which is usually randomly initialized subject to certain distributions, e.g., uniform distribution $\mc{U}(a, b)$ or normal distribution $\mc{N}(\mu, \sigma^2)$. To achieve better performance, the hyper-parameters $a$, $b$ in the uniform distribution, or $\mu$, $\sigma$ in the normal distribution can be fine-tuned.

\item Term $\nabla_{\boldsymbol{\theta}} \mc{L}(\boldsymbol{\theta}^{{(\tau-1)}}; \mc{T})$ denotes the derivative of the loss function $\mc{L}(\boldsymbol{\theta}; \mc{T})$ regarding the variable $\boldsymbol{\theta}$ based on the complete training set $\mc{T}$ and the variable vector value $\boldsymbol{\theta}^{{(\tau-1)}}$ achieved in iteration ${\tau-1}$.

\item Term $\eta$ denotes the \textit{learning rate} in the gradient descent algorithm, which is usually a hyper-parameter with a small value (e.g., $10^{-3}$). Fine-tuning of the parameter is necessary to achieve a fast convergence in practical applications of the \textit{vanilla gradient descent} algorithm.
\end{enumerate}

Such an iterative updating process continues until convergence, and the variable vector $\boldsymbol{\theta}$ achieved at convergence will be outputted as the (globally or locally) optimal variable for the deep learning models. The pseudo-code of the \textit{vanilla gradient descent} algorithm is available in Algorithm~\ref{alg:chapter_gradient_descent_GD}. In Algorithm~\ref{alg:chapter_gradient_descent_GD}, the model variables are initialized with the normal distribution. The normal distribution standard deviation $\sigma$ is the input parameters, and the mean $\mu$ is set as $0$. As to the the convergence condition, it can be (1) the loss term changes in two sequential iterations is less than a pre-specified threshold, (2) the model variable $\theta$ changes is within a pre-specified range, or (3) the pre-specified iteration round number has been reached. 

\begin{algorithm}[H]
\small
\caption{Vanilla Gradient Descent}
\label{alg:chapter_gradient_descent_GD}
\begin{algorithmic}[1]
	\REQUIRE Training Set: $\mc{T}$; Learning Rate $\eta$; Normal Distribution Std: $\sigma$.
\ENSURE  Model Parameter $\boldsymbol{\theta}$
\STATE	{Initialize parameter with Normal distribution $\boldsymbol{\theta} \sim N(0, \sigma^2)$}
\STATE	{Initialize convergence $tag = False$}
\WHILE	{$tag == False$}
\STATE	{Compute gradient $\nabla_{\boldsymbol{\theta}} \mc{L}(\boldsymbol{\theta}; \mc{T})$ on the training set $\mc{T}$}
\STATE	{Update variable $\boldsymbol{\theta} = \boldsymbol{\theta}- \eta \cdot \nabla_{\boldsymbol{\theta}} \mc{L}(\boldsymbol{\theta}; \mc{T})$}
\IF		{convergence condition holds}
\STATE	{$tag = True$}
\ENDIF
\ENDWHILE
\STATE	{\textbf{Return} model variable $\boldsymbol{\theta}$}
\end{algorithmic}
\end{algorithm}

According to the descriptions, in the \textit{vanilla gradient descent} algorithm, to update the model variables, we need to compute the gradient of the loss function regarding the variable, i.e., $\nabla_{\boldsymbol{\theta}} \mc{L}(\boldsymbol{\theta}^{{(\tau-1)}}; \mc{T})$, in each iteration, which is usually very time consuming for a large training set. To improve the learning efficiency, we will introduce the \textit{stochastic gradient descent} and \textit{mini-batch gradient descent} learning algorithms in the following two subsections.


\subsection{Stochastic Gradient Descent}\label{ref:subsec_SGD}

\textit{Vanilla gradient descent} computes loss function gradient with the complete training set, which will be very inefficient if the training set contains lots of data instances. Instead of using the complete training set, the \textit{stochastic gradient descent} (SGD) algorithm updates the model variables by computing the loss function gradient instances by instances. Therefore, \textit{stochastic gradient descent} can be much faster than \textit{vanilla gradient descent}, but it may also cause heavy fluctuations of the loss function in the updating process.

Formally, given the training set $\mc{T}$ and the initialized model variable vector $\boldsymbol{\theta}^{(0)}$, for each instance $(\mb{x}_i, \mb{y}_i) \in \mc{T}$, \textit{stochastic gradient descent} algorithm will update the model variable with the following equation iteratively:
\begin{equation}
\boldsymbol{\theta}^{(\tau)}  = \boldsymbol{\theta}^{{(\tau-1)}} - \eta \cdot \nabla_{\boldsymbol{\theta}} \mc{L}(\boldsymbol{\theta}^{{(\tau-1)}}; (\mb{x}_i, \mb{y}_i)),
\end{equation}
where the loss term $\mc{L}(\boldsymbol{\theta}; (\mb{x}_i, \mb{y}_i)) = \ell (F(\mb{x}_i; \boldsymbol{\theta}), \mb{y}_i)$ denotes the loss introduced by the model on data instance $(\mb{x}_i, \mb{y}_i)$. The pseudo-code of the \textit{stochastic gradient descent} learning algorithm is available in Algorithm~\ref{alg:chapter_gradient_descent_SGD}

\begin{algorithm}[H]
\small
\caption{Stochastic Gradient Descent}
\label{alg:chapter_gradient_descent_SGD}
\begin{algorithmic}[1]
	\REQUIRE Training Set $\mc{T}$; Learning Rate $\eta$; Normal Distribution Std: $\sigma$.
\ENSURE  Model Parameter $\boldsymbol{\theta}$
\STATE	{Initialize parameter with Normal distribution $\boldsymbol{\theta} \sim N(0, \sigma^2)$}
\STATE	{Initialize convergence $tag = False$}
\WHILE	{$tag == False$}
\STATE	{Shuffle the training set $\mc{T}$}
\FOR	{each data instance $(\mb{x}_i, \mb{y}_i) \in \mc{T}$}
\STATE	{Compute gradient $\nabla_{\boldsymbol{\theta}} \mc{L}(\boldsymbol{\theta}; (\mb{x}_i, \mb{y}_i))$ on the training instance $(\mb{x}_i, \mb{y}_i)$}
\STATE	{Update variable $\boldsymbol{\theta} = \boldsymbol{\theta}- \eta \cdot \nabla_{\boldsymbol{\theta}} \mc{L}(\boldsymbol{\theta}; (\mb{x}_i, \mb{y}_i))$}
\ENDFOR
\IF		{convergence condition holds}
\STATE	{$tag = True$}
\ENDIF
\ENDWHILE
\STATE	{\textbf{Return} model variable $\boldsymbol{\theta}$}
\end{algorithmic}
\end{algorithm}

According to Algorithm~\ref{alg:chapter_gradient_descent_SGD}, the whole training set will be shuffled before the updating process. Even though the learning process of \textit{stochastic gradient descent} may fluctuate a lot, which actually also provides \textit{stochastic gradient descent} with the ability to jump out of local optimum. Meanwhile, by selecting a small learning rate $\eta$ and decreasing it in the learning process, \textit{stochastic gradient descent} can almost certainly converge to the local or global optimum. 


\subsection{Mini-batch Gradient Descent}\label{ref:subsec_Minibatch_GD}

To balance between the \textit{vanilla gradient descent} and \textit{stochastic gradient descent} learning algorithms, \textit{mini-batch gradient descent} proposes to update the model variables with a mini-batch of training instances instead. Formally, let $\mc{B} \subset \mc{T}$ denote the mini-batch of training instances sampled from the training set $\mc{T}$. We can represent the variable updating equation of the deep learning model with the mini-batch as follows:
\begin{equation}
\boldsymbol{\theta}^{(\tau)}  = \boldsymbol{\theta}^{{(\tau-1)}} - \eta \cdot \nabla_{\boldsymbol{\theta}} \mc{L}(\boldsymbol{\theta}^{{(\tau-1)}}; \mc{B}),
\end{equation}
where $\mc{L}(\boldsymbol{\theta}; \mc{B})$ denotes the loss term introduced by the model on the mini-batch $\mc{B}$. The pseudo-code of the \textit{mini-batch gradient descent} algorithm can be illustrated in Algorithm~\ref{alg:chapter_gradient_descent_mini-batch_GD}. 

\begin{algorithm}[H]
\small
\caption{Mini-batch Gradient Descent}
\label{alg:chapter_gradient_descent_mini-batch_GD}
\begin{algorithmic}[1]
	\REQUIRE Training Set $\mc{T}$; Learning Rate $\eta$; Normal Distribution Std $\sigma$; Mini-batch Size $b$.
\ENSURE  Model Parameter $\boldsymbol{\theta}$
\STATE	{Initialize parameter with Normal distribution $\boldsymbol{\theta} \sim N(0, \sigma^2)$}
\STATE	{Initialize convergence $tag = False$}
\WHILE	{$tag == False$}
\STATE	{Shuffle the training set $\mc{T}$}
\FOR	{each mini-batch $\mc{B} \subset \mc{T}$}
\STATE	{Compute gradient $\nabla_{\boldsymbol{\theta}} \mc{L}(\boldsymbol{\theta}; \mc{B})$ on the mini-batch $\mc{B}$}
\STATE	{Update variable $\boldsymbol{\theta} = \boldsymbol{\theta}- \eta \cdot \nabla_{\boldsymbol{\theta}} \mc{L}(\boldsymbol{\theta}; \mc{B})$}
\ENDFOR
\IF		{convergence condition holds}
\STATE	{$tag = True$}
\ENDIF
\ENDWHILE
\STATE	{\textbf{Return} model variable $\boldsymbol{\theta}$}
\end{algorithmic}
\end{algorithm}

In the \textit{mini-batch gradient descent} algorithm, the batch size $b$ is provided as a parameter, which can take values, like $64$, $128$ or $256$. In most of the cases, $b$ should not be very large and its specific value depends on the size of the training set a lot. The mini-batches are usually sampled sequentially from the training set $\mc{T}$. In other words, the training set $\mc{T}$ can be divided into multiple batches of size $b$, and these batches can be picked one by one for updating the model variables sequentially. In some versions of the \textit{mini-batch gradient descent}, instead of sequential batch selection, they propose to randomly sample the mini-batches from the training set, which is also referred to as the random mini-batch generation process.

Compared with \textit{vanilla gradient descent}, the \textit{mini-batch gradient descent} algorithm is much more efficient especially for the training set of an extremely large size. Meanwhile, compared with the \textit{stochastic gradient descent}, the \textit{mini-batch gradient descent} algorithm greatly reduces the variance in the model variable updating process and can achieve much more stable convergence.


\subsection{Performance Analysis of Vanilla GD, SGD and Mini-Batch GD}

These three \textit{gradient descent} algorithms introduced here work well for many optimization problems, and can all converge to a promising (local or global) optimum. However, these algorithms also suffer from several problems listed as follows:
\begin{itemize}

\item \textit{Learning Rate Selection}: The learning rate $\eta$ may affect the convergence of the gradient descent algorithms a lot. A large learning rate may diverge the learning process, while a small learning rate renders the convergence too slow. Therefore, selecting a good learning rate will be very important for the gradient descent algorithms.

\item \textit{Learning Rate Adjustment}: In most of the cases, a fixed learning rate in the whole updating process cannot work well for the gradient descent algorithms. In the initial stages, the algorithm may need a larger learning rate to reach a good (local or global) optimum fast. However, in the later stages, the algorithm may need to adjust the learning rate with a smaller value to fine-tune the performance instead.

\item \textit{Variable Individual Learning Rate}: For different variables, their update may actually require a different learning rates instead in the updating process. Therefore, how to use an individual learning rate for different variables is required and necessary.

\item \textit{Saddle Point Avoid}: Formally, the saddle point denotes the points with zero gradient in all the dimensions. However, for some of the dimensions, the saddle point is the local minimum, while for some other dimensions, the point is the local maximum. How to get out from such saddle points is a very challenging problem.

\end{itemize}

In the following part of this paper, we will introduce several other gradient descent based learning algorithm variants, which are mainly proposed to resolve one or several of the above problems.


\section{Momentum based Learning Algorithms}\label{ref:sec_Momentum}

In this section, we will introduce a group of gradient descent variant algorithms, which propose to update the model variables with both the current gradient as well as the historical gradients simultaneously, which are named as the \textit{momentum based gradient descent algorithm}. Here, these algorithms will be talked about by using the mini-batch GD as the base learning algorithm.

\subsection{Momentum}\label{ref:subsec_MGD}

To smooth the fluctuation encountered in the learning process for the gradient descent algorithms (e.g., mini-batch gradient descent), Momentum \cite{momentum} is proposed to accelerate the updating convergence of the variables. Formally, Momentum updates the variables with the following equation:
\begin{alignat}{2}\label{equ:momentum_updating_equation}
\boldsymbol{\theta}^{(\tau)} &= \boldsymbol{\theta}^{(\tau - 1)} - \eta \cdot \Delta \mb{v}^{(\tau)}
\mbox{ where }
\Delta \mb{v}^{(\tau)} &= \rho \cdot \Delta\mb{v}^{(\tau-1)} + (1- \rho) \cdot \nabla_{\boldsymbol{\theta}} \mc{L}(\boldsymbol{\theta}^{(\tau-1)}),
\end{alignat}
In the above equation, $\mb{v}^{(\tau)} $ denotes the momentum term introduced for keeping record of the historical gradients till iteration $\tau$ and function $\mc{L}(\boldsymbol{\theta}) = \mc{L}(\boldsymbol{\theta}; \mc{B})$ denotes the loss function on a mini-batch $\mc{B}$. Parameter $\rho \in [0,1]$ denotes the weight of the momentum term. The pseudo-code of the Momentum based gradient descent learning algorithm is available in Algorithm~\ref{alg:chapter_gradient_descent_Momentum}.

\begin{algorithm}[H]
\small
\caption{Momentum based Mini-batch Gradient Descent}
\label{alg:chapter_gradient_descent_Momentum}
\begin{algorithmic}[1]
	\REQUIRE Training Set $\mc{T}$; Learning Rate $\eta$; Normal Distribution Std $\sigma$; Mini-batch Size $b$; Momentum Term Weight: $\rho$.
\ENSURE  Model Parameter $\boldsymbol{\theta}$
\STATE	{Initialize parameter with Normal distribution $\boldsymbol{\theta} \sim N(0, \sigma^2)$}
\STATE	{Initialize Momentum term $\Delta \mb{v} = \mb{0}$}
\STATE	{Initialize convergence $tag = False$}
\WHILE	{$tag == False$}
\STATE	{Shuffle the training set $\mc{T}$}
\FOR	{each mini-batch $\mc{B} \subset \mc{T}$}
\STATE	{Compute gradient $\nabla_{\boldsymbol{\theta}} \mc{L}(\boldsymbol{\theta}; \mc{B})$ on the mini-batch $\mc{B}$}
\STATE	{Update term $\Delta \mb{v} = \rho \cdot \Delta \mb{v} + (1-\rho) \cdot  \nabla_{\boldsymbol{\theta}} \mc{L}(\boldsymbol{\theta}; \mc{B})$}
\STATE	{Update variable $\boldsymbol{\theta} = \boldsymbol{\theta}- \eta \cdot \Delta \mb{v}$}
\ENDFOR
\IF		{convergence condition holds}
\STATE	{$tag = True$}
\ENDIF
\ENDWHILE
\STATE	{\textbf{Return} model variable $\boldsymbol{\theta}$}
\end{algorithmic}
\end{algorithm}

Based on such a recursive updating equation of vector $\boldsymbol{\theta}^{{(\tau)} }$, we can actually achieve an equivalent representation of $\boldsymbol{\theta}^{{(\tau)} }$ merely with the gradient terms of the loss function according to the following Lemma.

\begin{lemma}\label{lemma:momentum_updating_equation}
Vector $\Delta \mb{v}^{(\tau)}$ can be formally represented with the following equation:
\begin{equation}
\Delta \mb{v}^{(\tau)} = \sum_{t = 0}^{\tau - 1} \rho^t \cdot (1-\rho) \cdot \nabla_{\boldsymbol{\theta}} \mc{L}(\boldsymbol{\theta}^{(\tau-1-t)}).
\end{equation}
\end{lemma}

\begin{proof}
The Lemma can be proved by induction on $\tau$.
\begin{enumerate}
\item Base Case: In the case that $\tau = 1$, according to the recursive representation of  $\boldsymbol{\theta}$, we have
\begin{align}
\begin{aligned}
\Delta \mb{v}^{(1)} &= \rho \cdot \Delta\mb{v}^{(0)} + (1- \rho) \cdot \nabla_{\boldsymbol{\theta}} \mc{L}(\boldsymbol{\theta}^{(0)})\\
&= (1- \rho) \cdot \nabla_{\boldsymbol{\theta}} \mc{L}(\boldsymbol{\theta}^{(0)}),
\end{aligned}
\end{align} 
where $\Delta\mb{v}^{(0)}$ is initialized as a zero vector.

\item Assumption: We assume the equation holds for $\tau = k$, i.e., 
\begin{equation}
\Delta \mb{v}^{(k)} = \sum_{t = 0}^{k - 1} \rho^t \cdot (1-\rho) \cdot \nabla_{\boldsymbol{\theta}} \mc{L}(\boldsymbol{\theta}^{(k-1-t)}).
\end{equation}

\item Induction: For $\tau = k+1$, based on Equation~(\ref{equ:momentum_updating_equation}), we can represent the vector $\Delta \mb{v}^{(k+1)}$ to be
\begin{align}
\begin{aligned}
\Delta \mb{v}^{(k+1)} &= \rho \cdot \Delta\mb{v}^{(k)} + (1- \rho) \cdot \nabla_{\boldsymbol{\theta}} \mc{L}(\boldsymbol{\theta}^{(k)})\\
&= \rho \cdot \sum_{t = 0}^{k - 1} \rho^t \cdot (1-\rho) \cdot \nabla_{\boldsymbol{\theta}} \mc{L}(\boldsymbol{\theta}^{(k-1-t)}) +  (1- \rho) \cdot \nabla_{\boldsymbol{\theta}} \mc{L}(\boldsymbol{\theta}^{(k)})\\
&= \sum_{t = 1}^{k} \rho^t \cdot (1-\rho) \cdot \nabla_{\boldsymbol{\theta}} \mc{L}(\boldsymbol{\theta}^{(k-t)}) +  (1- \rho) \cdot \nabla_{\boldsymbol{\theta}} \mc{L}(\boldsymbol{\theta}^{(k)})\\
&= \sum_{t = 0}^{k} \rho^t \cdot (1-\rho) \cdot \nabla_{\boldsymbol{\theta}} \mc{L}(\boldsymbol{\theta}^{(k-t)}),
\end{aligned}
\end{align}
which can conclude this proof.
\end{enumerate}
\end{proof}

Considering the parameter $\rho \in [0, 1]$, the gradients of the iterations which are far away from the current iteration will decay exponentially. Compared with the conventional gradient descent algorithm, the Momentum can achieve a faster convergence rate, which can be explained based on the gradient updated in the above equation. For the gradient descent algorithm, the gradient updated in iteration $\tau$ can be denoted as $\nabla_{\boldsymbol{\theta}} \mc{L}(\boldsymbol{\theta}^{(\tau-1)})$, while the updated gradient for Momentum can be denoted as $\Delta\mb{v}^{(\tau)} = (1-\rho) \cdot \nabla_{\boldsymbol{\theta}} \mc{L}(\boldsymbol{\theta}^{(\tau-1)}) + \rho \cdot \Delta\mb{v}^{(\tau-1)}$. Term $\Delta\mb{v}^{(\tau-1)}$ keeps records of the historical gradients. As the gradient descent algorithms updates the variables from the steep zone to the relatively flat zone, the large historical gradient terms stored in $\Delta\mb{v}^{(\tau-1)}$ allows will actually accelerate the algorithm to reach the (local or global) optimum. 


\subsection{Nesterov Accelerated Gradient}\label{ref:subsec_NAG}

For the gradient descent algorithms introduced so far, they all update the variables based on the gradients at both the historical or current points without knowledge about the future points they will reach in the learning process. It makes the learning process to be blind and the learning performance highly unpredictable. 

The \textit{Nesterov Accelerated Gradient} (i.e., NAG) \cite{nag} method propose to resolve this problem by updating the variables with the gradient at an approximated future point instead. The variable $\boldsymbol{\theta}$ to be achieved at step $\tau$ can be denoted as $\boldsymbol{\theta}^{(\tau)}$, and it can be estimated as $\hat{\boldsymbol{\theta}}^{(\tau)} = \boldsymbol{\theta}^{(\tau-1)} - \eta \cdot \rho \cdot \Delta \mb{v}^{(\tau-1)}$, where $\Delta \mb{v}^{(\tau-1)}$ denotes the momentum term at iteration $\tau -1$. Formally, assisted with the \textit{look-ahead gradient} and \textit{momentum}, the variable updating equations in NAG can be formally represented as follows:
\begin{alignat}{2}\label{equ:NAG_updating_equation}
\boldsymbol{\theta}^{(\tau)} = \boldsymbol{\theta}^{(\tau - 1)} - \eta \cdot \Delta \mb{v}^{(\tau)},
\mbox{ where }
\begin{cases}
\hat{\boldsymbol{\theta}}^{(\tau)} &= {\boldsymbol{\theta}}^{(\tau-1)} - \eta \cdot \rho \cdot \Delta \mb{v}^{(\tau-1)},\\
\Delta \mb{v}^{(\tau)} &= \rho \cdot \Delta\mb{v}^{(\tau-1)} + (1- \rho) \cdot \nabla_{\boldsymbol{\theta}} \mc{L}(\hat{\boldsymbol{\theta}}^{(\tau)}).
\end{cases}
\end{alignat}

By computing the gradient one step forward, the NAG algorithm is able to update the model variables much more effectively. The pseudo-code of the NAG learning algorithm is available in Algorithm~\ref{alg:chapter_gradient_descent_NAG}.

\begin{algorithm}[H]
\small
\caption{NAG based Mini-batch Gradient Descent}
\label{alg:chapter_gradient_descent_NAG}
\begin{algorithmic}[1]
	\REQUIRE Training Set $\mc{T}$; Learning Rate $\eta$; Normal Distribution Std $\sigma$; Mini-batch Size $b$; Momentum Term Weight: $\rho$.
\ENSURE  Model Parameter $\boldsymbol{\theta}$
\STATE	{Initialize parameter with Normal distribution $\boldsymbol{\theta} \sim N(0, \sigma^2)$}
\STATE	{Initialize Momentum term $\Delta \mb{v} = \mb{0}$}
\STATE	{Initialize convergence $tag = False$}
\WHILE	{$tag == False$}
\STATE	{Shuffle the training set $\mc{T}$}
\FOR	{each mini-batch $\mc{B} \subset \mc{T}$}
\STATE	{Compute $\hat{\boldsymbol{\theta}} = {\boldsymbol{\theta}} - \eta \cdot \rho \cdot \Delta \mb{v}$}
\STATE	{Compute gradient $\nabla_{\boldsymbol{\theta}} \mc{L}(\hat{\boldsymbol{\theta}}; \mc{B})$ on the mini-batch $\mc{B}$}
\STATE	{Update term $\Delta \mb{v} = \rho \cdot \Delta \mb{v} + (1-\rho) \cdot  \nabla_{\boldsymbol{\theta}} \mc{L}(\hat{\boldsymbol{\theta}}; \mc{B})$}
\STATE	{Update variable $\boldsymbol{\theta} = \boldsymbol{\theta}- \eta \cdot \Delta \mb{v}$}
\ENDFOR
\IF		{convergence condition holds}
\STATE	{$tag = True$}
\ENDIF
\ENDWHILE
\STATE	{\textbf{Return} model variable $\boldsymbol{\theta}$}
\end{algorithmic}
\end{algorithm}


\section{Adaptive Gradient based Learning Algorithms}\label{ref:sec_Adaptive}

In this section, we will introduce a group of learning algorithms which updates the variables with adaptive learning rates. The algorithms introduced here include Adagrad, RMSprop and Adadelta respectively.

\subsection{Adagrad}\label{ref:subsec_Adagrad}

For the learning algorithms introduced before, the learning rate in these methods are mostly fixed and identical for all the variables in vector $\boldsymbol{\theta}$. However, in the variable updating process, for different variables, their required learning rates can be different. For the variables reaching the optimum, a smaller learning rate is needed; while for the variables far away from the optimum, we may need to use a relatively larger learning rate so as to reach the optimum faster. To resolve these two problems, in this part, we will introduce one \textit{adaptive gradient method}, namely Adagrad \cite{adagrad}. 

Formally, the learning equation for Adagrad can be represented as follows:
\begin{alignat}{2}
\boldsymbol{\theta}^{(\tau)} = \boldsymbol{\theta}^{(\tau - 1)} - \frac{\eta}{\sqrt{ diag(\mb{G}^{(\tau)}) + \epsilon \cdot \mb{I} }} \mb{g}^{(\tau - 1)}
\mbox{ where }
\begin{cases}
\mb{g}^{(\tau - 1)} &= \nabla_{\boldsymbol{\theta}} \mc{L}({\boldsymbol{\theta}}^{(\tau-1)}),\\
\mb{G}^{(\tau)} &= \sum_{t=0}^{\tau-1} \mb{g}^{(t)} (\mb{g}^{(t)})^\top.
\end{cases}
\end{alignat}
In the above updating equation, operator $diag(\mb{G})$ defines a diagonal matrix of the same dimensions as $\mb{G}$ but only with the elements on the diagonal of $\mb{G}$, and $\epsilon$ is a smoothing term to avoid the division of zero values on the diagonal of the matrix. Matrix $\mb{G}^{(\tau)}$ keeps records of the computed historical gradients from the beginning until the current iteration. Considering the values on the diagonal of matrix $\mb{G}^{(\tau)}$ will be different, the learning rate of variables in vector $\boldsymbol{\theta}$ will be different, e.g., $\eta_i^{(\tau)} = \frac{\eta}{\sqrt{\mb{G}^{(\tau)}(i,i) + \epsilon}}$ for variable $\boldsymbol{\theta}(i)$ in iteration $\tau$. In addition, since the matrix $\mb{G}^{(\tau)}$ will be updated in each iteration, the learning rate for the same variable in different iterations will be different as well, which is the reason why the algorithm is called the adaptive learning algorithms. The pseudo-code of Adagrad is provided in Algorithm~\ref{alg:chapter_gradient_descent_Adagrad}. 

\begin{algorithm}[H]
\small
\caption{Adagrad based Mini-batch Gradient Descent}
\label{alg:chapter_gradient_descent_Adagrad}
\begin{algorithmic}[1]
	\REQUIRE Training Set $\mc{T}$; Learning Rate $\eta$; Normal Distribution Std $\sigma$; Mini-batch Size $b$.
\ENSURE  Model Parameter $\boldsymbol{\theta}$
\STATE	{Initialize parameter with Normal distribution $\boldsymbol{\theta} \sim N(0, \sigma^2)$}
\STATE	{Initialize Matrix $\mb{G} = \mb{0}$}
\STATE	{Initialize convergence $tag = False$}
\WHILE	{$tag == False$}
\STATE	{Shuffle the training set $\mc{T}$}
\FOR	{each mini-batch $\mc{B} \subset \mc{T}$}
\STATE	{Compute gradient  vector $\mb{g} = \nabla_{\boldsymbol{\theta}} \mc{L}({\boldsymbol{\theta}}; \mc{B})$ on the mini-batch $\mc{B}$}
\STATE	{Update matrix $\mb{G} = \mb{G} + \mb{g}\mb{g}^\top$}
\STATE	{Update variable $\boldsymbol{\theta} = \boldsymbol{\theta}- \frac{\eta}{ \sqrt{ diag(\mb{G}) + \epsilon \cdot \mb{I} } } \mb{g}$}
\ENDFOR
\IF		{convergence condition holds}
\STATE	{$tag = True$}
\ENDIF
\ENDWHILE
\STATE	{\textbf{Return} model variable $\boldsymbol{\theta}$}
\end{algorithmic}
\end{algorithm}

Such a learning mechanism in Adagrad allows both adaptive learning rate in the learning process without the need of manual tuning, as well as different learning rate for different variables. Furthermore, as the iteration continues, the values on the diagonal of matrix $\mb{G}$ will be no decreasing, i.e., the learning rates of the variables will keep decreasing in the learning process. It will also create problems for the learning in later iterations, since the variables can no longer be effectively updated with information from the training data. In addition, Adagrad still needs an initial learning rate parameter $\eta$ to start the learning process, which can also be treated as one of the shortcomings of Adagrad.


\subsection{RMSprop}\label{ref:subsec_RMSprop}

To resolve the problem with monotonically decreasing learning rate in Adagrad (i.e., entries in matrix $\frac{\eta}{\sqrt{ diag(\mb{G}^{(\tau)}) + \epsilon \cdot \mb{I} }} $), in this part, we will introduce another learning algorithm, namely RMSprop \cite{rmsprop}. RMSprop properly decays the weight of the historical accumulated gradient in the matrix $\mb{G}$ defined in Adagrad, and also allows the adjustment of learning rate in the updating process. Formally, the variable updating equations in RMSprop can be represented with the following equations:
\begin{equation}
\boldsymbol{\theta}^{(\tau)} = \boldsymbol{\theta}^{(\tau - 1)} - \frac{\eta}{\sqrt{ diag(\mb{G}^{(\tau)}) + \epsilon \cdot \mb{I} }} \mb{g}^{(\tau - 1)},
\end{equation}
where
\begin{equation}\label{equ:RMSprop_Adaptive}
\begin{cases}
\mb{g}^{(\tau - 1)} &= \nabla_{\boldsymbol{\theta}} \mc{L}({\boldsymbol{\theta}}^{(\tau-1)}),\\
\mb{G}^{(\tau)} &= \rho \cdot \mb{G}^{(\tau-1)} + (1-\rho) \cdot \mb{g}^{(\tau-1)} (\mb{g}^{(\tau-1)})^\top.
\end{cases}
\end{equation}

In the above equation, the denominator term is also usually denoted as $RMS(\mb{g}^{(\tau-1)}) =\sqrt{ diag(\mb{G}^{(\tau)}) + \epsilon \cdot \mb{I} }$ (RMS denotes \textit{root mean square} metric on vector $\mb{g}^{(\tau-1)}$). Therefore, the updating equation is also usually written as follows:
\begin{equation}
\boldsymbol{\theta}^{(\tau)} = \boldsymbol{\theta}^{(\tau - 1)} - \frac{\eta}{ RMS(\mb{g}^{(\tau -1)}) } \mb{g}^{(\tau - 1)}.
\end{equation}
In the representation of matrix $\mb{G}$, parameter $\rho$ denotes the weight of the historically accumulated gradients, and $\rho$ is usually set as $0.9$ according to Geoff Hinton in his Lecture at Coursera. Based on the representation of $\mb{G}$, for the gradient computed in $t$ iterations ahead of the current iteration, they will be assigned with a very small weight, i.e., $\rho^{(t)} \cdot (1-\rho)$, which decays exponentially with $t$. The pseudo-code of RMSprop is provided in Algorithm~\ref{alg:chapter_gradient_descent_RMSprop}, where the learning rate $\eta$ is still needed and provided as a parameter in the algorithm.

\begin{algorithm}[H]
\small
\caption{RMSprop based Mini-batch Gradient Descent}
\label{alg:chapter_gradient_descent_RMSprop}
\begin{algorithmic}[1]
	\REQUIRE Training Set $\mc{T}$; Learning Rate $\eta$; Normal Distribution Std $\sigma$; Mini-batch Size $b$; Parameter $\rho$.
\ENSURE  Model Parameter $\boldsymbol{\theta}$
\STATE	{Initialize parameter with Normal distribution $\boldsymbol{\theta} \sim N(0, \sigma^2)$}
\STATE	{Initialize Matrix $\mb{G} = \mb{0}$}
\STATE	{Initialize convergence $tag = False$}
\WHILE	{$tag == False$}
\STATE	{Shuffle the training set $\mc{T}$}
\FOR	{each mini-batch $\mc{B} \subset \mc{T}$}
\STATE	{Compute gradient  vector $\mb{g} = \nabla_{\boldsymbol{\theta}} \mc{L}({\boldsymbol{\theta}}; \mc{B})$ on the mini-batch $\mc{B}$}
\STATE	{Update matrix $\mb{G} = \rho \cdot \mb{G} +  (1-\rho) \cdot \mb{g}\mb{g}^\top$}
\STATE	{Update variable $\boldsymbol{\theta} = \boldsymbol{\theta}- \frac{\eta}{ \sqrt{ diag(\mb{G}) + \epsilon \cdot \mb{I} } } \mb{g}$}
\ENDFOR
\IF		{convergence condition holds}
\STATE	{$tag = True$}
\ENDIF
\ENDWHILE
\STATE	{\textbf{Return} model variable $\boldsymbol{\theta}$}
\end{algorithmic}
\end{algorithm}


\subsection{Adadelta}\label{ref:subsec_Adadelta}

Adadelta \cite{adadelta} and RMSprop are proposed separately by different people almost at the same time, which address the monotonically decreasing learning rate in Adagrad with identical methods. Meanwhile, compared with RMSprop further improves Adagrad by introducing a mechanism to eliminate the learning rate from the updating equations of the variables. Based on the following updating equation we introduce before in RMSprop:
\begin{equation}
\boldsymbol{\theta}^{(\tau)} = \boldsymbol{\theta}^{(\tau - 1)} - \frac{\eta}{ RMS(\mb{g}^{(\tau -1)}) } \mb{g}^{(\tau - 1)}.
\end{equation}

However, as introduced in \cite{adadelta}, the learning algorithms introduced so far fail to consider the units of the learning variables in the updating process. Here, the units effectively indicate the physical meanings of the variables, which can be ``\textit{km}'', ``\textit{s}'' or ``\textit{kg}''. If the parameter $\boldsymbol{\theta}$ has hypothetical units, then the updates in parameter, i.e., $\Delta \boldsymbol{\theta} = \nabla_{\boldsymbol{\theta}} \mc{L}(\boldsymbol{\theta})$, should have the same units as well. However, for the learning algorithms we have introduced before, e.g., SGD, Momentum, NAD, and Adagrad, such an assumption cannot hold. For instance, for the case of SGD, the units of update term is actually proportional to the inverse of the units of $\boldsymbol{\theta}$:
\begin{equation}
\mbox{units of } \Delta \boldsymbol{\theta} \propto \mbox{units of } \mb{g} \propto \mbox{units of } \frac{\partial \mathcal{L}(\cdot)}{\partial \boldsymbol{\theta}} \propto \frac{1}{\mbox{units of } \boldsymbol{\theta}}.
\end{equation}

To handle such a problem, Adadelta proposes to look at the second-order methods, e.g., Newton's method that uses Hessian approximation. In Newton's method, the units of variables updated can be denoted as 
\begin{equation}
\mbox{units of } \Delta \boldsymbol{\theta} \propto \mbox{units of } \mb{H}^{-1} \mb{g} \propto \mbox{units of } \frac{ \frac{ \partial \mc{L}(\cdot) }{ \partial \boldsymbol{\theta} } }{ \frac{ \partial^2 \mc{L}(\cdot) }{ \partial \boldsymbol{\theta}^2 } } \propto {\mbox{units of } \boldsymbol{\theta}},
\end{equation}
where $\mb{H} = \frac{\partial^2 \mc{L}(\boldsymbol{\theta})}{ \partial \boldsymbol{\theta}^2 }$ denotes the Hessian matrix computed with the derivatives of the loss function regarding the variables.

To match the units, Adadelta rearranges the Newton's method updating term as
\begin{equation}
\Delta \boldsymbol{\theta} = - \frac{ \frac{ \partial \mc{L}(\cdot) }{ \partial \boldsymbol{\theta} } }{ \frac{ \partial^2 \mc{L}(\cdot) }{ \partial \boldsymbol{\theta}^2 } },
\end{equation}
from which we can derive 
\begin{equation}
\mb{H}^{-1} = - \frac{1}{\frac{\partial^2 \mc{L}(\boldsymbol{\theta})}{\partial \boldsymbol{\theta}^2}} = - \frac{\Delta \boldsymbol{\theta}}{\frac{ \partial \mc{L}(\boldsymbol{\theta}) }{ \partial \boldsymbol{\theta} }}.
\end{equation}

Therefore, we can rewrite the updating equation for the variables based on Newton's method as follows:
\begin{align}
\begin{aligned}
\boldsymbol{\theta}^{(\tau)} &= \boldsymbol{\theta}^{(\tau - 1)} - (\mb{H}^{(\tau)})^{-1} \mb{g}^{(\tau - 1)}\\
&= \boldsymbol{\theta}^{(\tau - 1)} - \frac{\Delta \boldsymbol{\theta}^{(\tau)}}{\frac{ \partial \mc{L}(\boldsymbol{\theta}^{(\tau)}) }{ \partial \boldsymbol{\theta} }} \mb{g}^{(\tau - 1)}.
\end{aligned}
\end{align}

Similar to RMSprop, Adadelta approximates the denominator in the above equation with the RMS of the previous gradients. Meanwhile, the $\Delta \boldsymbol{\theta}^{(\tau)}$ term in the current iteration is unknown yet. Adadelta proposes to approximate the numerator in the above equation, and adopts a similar way to use $RMS(\Delta \boldsymbol{\theta})$ to replace $\Delta \boldsymbol{\theta}$ instead. Formally, the variable updating equation in Adadelta can be formally written as follows:
\begin{equation}
\boldsymbol{\theta}^{(\tau)} = \boldsymbol{\theta}^{(\tau - 1)} - \frac{ RMS(\Delta \boldsymbol{\theta}^{(\tau -1)}) }{ RMS(\mb{g}^{(\tau -1)}) } \mb{g}^{(\tau - 1)},
\end{equation}
where $RMS(\Delta \boldsymbol{\theta}^{(\tau -1)})$ keep records of the $\Delta \boldsymbol{\theta}$ in the prior iterations until the previous iteration $\tau -1$. The pseudo-code of the Adadelta algorithm is provided in Algorithm~\ref{alg:chapter_gradient_descent_Adadelta}. According to the algorithm description, Adadelta doesn't use any learning rate in the updating equation, which effectively resolves the two weakness of Adagrad introduced before. 

\begin{algorithm}[H]
\small
\caption{Adadelta based Mini-batch Gradient Descent}
\label{alg:chapter_gradient_descent_Adadelta}
\begin{algorithmic}[1]
	\REQUIRE Training Set $\mc{T}$; Learning Rate $\eta$; Normal Distribution Std $\sigma$; Mini-batch Size $b$; Parameter $\rho$.
\ENSURE  Model Parameter $\boldsymbol{\theta}$
\STATE	{Initialize parameter with Normal distribution $\boldsymbol{\theta} \sim N(0, \sigma^2)$}
\STATE	{Initialize Matrix $\mb{G} = \mb{0}$}
\STATE	{Initialize Matrix $\mb{\Theta}= \mb{0}$}
\STATE	{Initialize convergence $tag = False$}
\WHILE	{$tag == False$}
\STATE	{Shuffle the training set $\mc{T}$}
\FOR	{each mini-batch $\mc{B} \subset \mc{T}$}
\STATE	{Compute gradient  vector $\mb{g} = \nabla_{\boldsymbol{\theta}} \mc{L}({\boldsymbol{\theta}}; \mc{B})$ on the mini-batch $\mc{B}$}
\STATE	{Update matrix $\mb{G} = \rho \cdot \mb{G} +  (1-\rho) \cdot \mb{g}\mb{g}^\top$}
\STATE	{Computer updating vector $\Delta \boldsymbol{\theta} = - \frac{\sqrt{ diag(\mb{\Theta}) + \epsilon \cdot \mb{I} }}{ \sqrt{ diag(\mb{G}) + \epsilon \cdot \mb{I} } } \mb{g}$}
\STATE 	{Update matrix $\mb{\Theta} = \rho \cdot \mb{\Theta} +  (1-\rho) \cdot \Delta \boldsymbol{\theta} (\Delta \boldsymbol{\theta})^\top$}
\STATE	{Update variable $\boldsymbol{\theta} = \boldsymbol{\theta} + \Delta \boldsymbol{\theta}$}
\ENDFOR
\IF		{convergence condition holds}
\STATE	{$tag = True$}
\ENDIF
\ENDWHILE
\STATE	{\textbf{Return} model variable $\boldsymbol{\theta}$}
\end{algorithmic}
\end{algorithm}


\section{Momentum \& Adaptive Gradient based Learning Algorithms}\label{ref:sec_Adam}

In this section, we will introduce the learning algorithms which combine the advantages of both the Momentum algorithm and the algorithm with adaptive learning rate, including Adam and Nadam respectively.

\subsection{Adam}\label{ref:subsec_Adam}

In recent years, a new learning algorithm, namely \textit{Adaptive Moment Estimation} (Adam) \cite{adam}, has been introduced, which only computes the first-order gradients with little memory requirement. Similar to RMSprop and Adadelta, Adam keeps records of the past squared first-order gradients; while Adam also keeps records of the past first-order gradients as well, both of which will decay exponentially in the learning process. Formally, we can use vectors $\mb{m}^{(\tau)}$ and $\mb{v}^{(\tau)}$ to denote the terms storing the first-order gradients and the squared first-order gradients respectively, whose concrete representations are provided as follows:
\begin{align}
\begin{aligned}
\mb{m}^{(\tau)} &= \beta_1 \cdot \mb{m}^{(\tau -1)} + (1- \beta_1) \cdot \mb{g}^{(\tau-1)},\\
\mb{v}^{(\tau)} &= \beta_2 \cdot \mb{v}^{(\tau -1)} + (1- \beta_2) \cdot \mb{g}^{(\tau-1)} \odot \mb{g}^{(\tau-1)},
\end{aligned}
\end{align}
where vector $\mb{g}^{(\tau-1)} = \nabla_{\boldsymbol{\theta}} \mc{L}({\boldsymbol{\theta}}^{(\tau-1)})$ and $\mb{g}^{(\tau-1)} \odot \mb{g}^{(\tau-1)}$ denote the element-wise product of the vectors. In the above equation, vector $\mb{v}^{(\tau)}$ actually stores the identical information as $diag(\mb{G}^{(\tau)})$ used in Equation~(\ref{equ:RMSprop_Adaptive}). However, instead of storing the whole matrix, Adam tends to use less space by keeping such a vector record.

As introduced in \cite{adam}, vectors $\mb{m}^{(\tau)}$ and $\mb{v}^{(\tau)}$ are biased toward zero, especially when $\beta_1$ and $\beta_2$ are close to $1$, since $\mb{m}^{(0)}$ and $\mb{v}^{(0)}$ are initialized to be the zero vectors respectively. To resolve such a problem, Adam introduces to rescale the terms as follows:
\begin{align}
\begin{aligned}
\hat{\mb{m}}^{(\tau)} = \frac{ \mb{m}^{(\tau)} }{1 - \beta_1^{\tau}},\\
\hat{\mb{v}}^{(\tau)} = \frac{ \mb{v}^{(\tau)} }{1 - \beta_2^{\tau}},
\end{aligned}
\end{align}
where the superscripts $\tau$ of terms $\beta_1^{\tau}$ and $\beta_2^{\tau}$ denotes the power instead of the iteration count index ${(\tau)}$ used before.

Based on the rescaled vector $\hat{\mb{m}}^{(\tau)}$ and matrix $\hat{\mb{v}}^{(\tau)}$, Adam will update the model variable with the following equation:
\begin{equation}\label{equ:Adam_updating_equation}
\boldsymbol{\theta}^{(\tau)} = \boldsymbol{\theta}^{(\tau-1)} - \frac{\eta}{\sqrt{\hat{\mb{v}}^{(\tau)}} + \epsilon} \odot \hat{\mb{m}}^{(\tau)}
\end{equation}

According to the above descriptions, Adam can be viewed as an integration of the RMSprop algorithm with the Momentum algorithm, which allows both faster convergence and adaptive learning rate simultaneously. Formally, the pseudo-code of the Adam learning algorithm is provided in Algorithm~\ref{alg:chapter_gradient_descent_Adam}, where $\beta_1$ and $\beta_2$ are inputted as the parameters for the algorithm. According to \cite{adam}, the parameters in Adam can be initialized as follows: $\epsilon = 10^{-8}$, $\beta_1 = 0.9$ and $\beta_2 = 0.999$.

\begin{algorithm}[H]
\small
\caption{Adam}
\label{alg:chapter_gradient_descent_Adam}
\begin{algorithmic}[1]
	\REQUIRE Training Set $\mc{T}$; Learning Rate $\eta$; Normal Distribution Std $\sigma$; Mini-batch Size $b$; Decay Parameters $\beta_1$, $\beta_2$.
\ENSURE  Model Parameter $\boldsymbol{\theta}$
\STATE	{Initialize parameter with Normal distribution $\boldsymbol{\theta} \sim N(0, \sigma^2)$}
\STATE	{Initialize vector $\mb{m} = \mb{0}$}
\STATE	{Initialize vector $\mb{v} = \mb{0}$}
\STATE	{Initialize step $\tau = {0}$}
\STATE	{Initialize convergence $tag = False$}
\WHILE	{$tag == False$}
\STATE	{Shuffle the training set $\mc{T}$}
\FOR	{each mini-batch $\mc{B} \subset \mc{T}$}
\STATE	{Update step $\tau=\tau+1$}
\STATE	{Compute gradient vector $\mb{g} = \nabla_{\boldsymbol{\theta}} \mc{L}({\boldsymbol{\theta}}; \mc{B})$ on the mini-batch $\mc{B}$}
\STATE	{Update vector $\mb{m} = \beta_1 \cdot \mb{m} + (1- \beta_1) \cdot \mb{g}$}
\STATE	{Update vector $\mb{v} = \beta_2 \cdot \mb{v} + (1- \beta_2) \cdot \mb{g} \odot \mb{g}$}
\STATE	{Rescale vector $\hat{\mb{m}} = \mb{m}/(1 - \beta_1^{\tau})$}
\STATE	{Rescale vector $\hat{\mb{v}} = \mb{v}/(1 - \beta_2^{\tau})$}
\STATE	{Update variable $\boldsymbol{\theta} = \boldsymbol{\theta} -  \frac{\eta}{\sqrt{\hat{\mb{v}}} + \epsilon } \odot \hat{\mb{m}}$}
\ENDFOR
\IF		{convergence condition holds}
\STATE	{$tag = True$}
\ENDIF
\ENDWHILE
\STATE	{\textbf{Return} model variable $\boldsymbol{\theta}$}
\end{algorithmic}
\end{algorithm}


\subsection{Nadam}\label{ref:subsec_Nadam}

Adam introduced in the previous section can be viewed as an integration of Momentum based learning algorithm with the adaptive gradient based learning algorithm, where the vanilla momentum is adopted. In \cite{nadam}, a learning algorithm is introduced to replace the vanilla momentum with the  \textit{Nesterov's accelerated gradient} (NAG) instead, and the new learning algorithm is called the \textit{Nesterov-accelerated Adam} (Nadam). Before providing the updating equation of Nadam, we would like to illustrate the updating equation of NAG as follows again (Equation~(\ref{equ:NAG_updating_equation})):
\begin{alignat}{2}
\boldsymbol{\theta}^{(\tau)} = \boldsymbol{\theta}^{(\tau - 1)} - \eta \cdot \Delta \mb{v}^{(\tau)},
\mbox{ where }
\begin{cases}
\Delta \mb{v}^{(\tau)} &= \rho \cdot \Delta\mb{v}^{(\tau-1)} + (1- \rho) \cdot \nabla_{\boldsymbol{\theta}} \mc{L}(\hat{\boldsymbol{\theta}}^{(\tau)}),\\
\hat{\boldsymbol{\theta}}^{(\tau)} &= {\boldsymbol{\theta}}^{(\tau-1)} - \eta \cdot \rho \cdot \Delta \mb{v}^{(\tau-1)}.
\end{cases}
\end{alignat}

According to the above updating equation, the momentum term $\Delta \mb{v}^{(\tau-1)}$ is used twice in the process: (1) $\Delta \mb{v}^{(\tau-1)}$ is used to compute $\hat{\boldsymbol{\theta}}^{(\tau)}$; and (2) $\Delta \mb{v}^{(\tau-1)}$ is used to compute $\Delta \mb{v}^{(\tau)}$. Nadam proposes to change the above method with the following updating equations instead:
\begin{equation}\label{equ:NAG_Momentum_in_Nadam}
\boldsymbol{\theta}^{(\tau)} = \boldsymbol{\theta}^{(\tau - 1)} - \eta \cdot \left(\rho \cdot \Delta \mb{v}^{(\tau)} + (1-\rho) \cdot \nabla_{\boldsymbol{\theta}} \mc{L}({\boldsymbol{\theta}}^{(\tau-1)}) \right),
\end{equation}
where
\begin{equation}
\Delta \mb{v}^{(\tau)} = \rho \cdot \Delta\mb{v}^{(\tau-1)} + (1- \rho) \cdot \nabla_{\boldsymbol{\theta}} \mc{L}({\boldsymbol{\theta}}^{(\tau-1)}).
\end{equation}

Compared with NAG, the above equation doesn't look-ahead in computing the gradient term; while compared with vanilla momentum, the above equation uses both the momentum term and the gradient term in the current iteration. Term $\rho \cdot \Delta \mb{v}^{(\tau)} + (1-\rho) \cdot \nabla_{\boldsymbol{\theta}} \mc{L}({\boldsymbol{\theta}}^{(\tau-1)})$ used in the updating equation actually is an approximation to $\Delta \mb{v}^{(\tau+1)}$ in the next iteration, which achieve the objective of looking ahead in updating the variables.

Meanwhile, according to Equation~(\ref{equ:Adam_updating_equation}), we can rewrite the updating equation of Adam as follows:
\begin{equation}\label{equ:Adam_updating_equation_in_Nadam}
\boldsymbol{\theta}^{(\tau)} = \boldsymbol{\theta}^{(\tau-1)} - \frac{\eta}{\sqrt{\hat{\mb{v}}^{(\tau)}} + \epsilon} \odot \hat{\mb{m}}^{(\tau)},
\end{equation}
where
\begin{equation}
\begin{aligned}
\hat{\mb{m}}^{(\tau)} = \frac{ \mb{m}^{(\tau)} }{1 - \beta_1^{\tau}}, \mbox{ and } \mb{m}^{(\tau)} = \beta_1 \cdot \mb{m}^{(\tau -1)} + (1- \beta_1) \cdot \mb{g}^{(\tau)}.
\end{aligned}
\end{equation}

By replacing the $\hat{\mb{m}}^{(\tau)}$ term into Equation~(\ref{equ:Adam_updating_equation_in_Nadam}), we can rewrite it as follows:
\begin{align}
\begin{aligned}
\boldsymbol{\theta}^{(\tau)} &= \boldsymbol{\theta}^{(\tau-1)} - \frac{\eta}{\sqrt{\hat{\mb{v}}^{(\tau)}} + \epsilon} \odot \left( \frac{ \beta_1 \cdot \mb{m}^{(\tau -1)} }{1 - \beta_1^{\tau}} + \frac{ (1- \beta_1) \cdot \mb{g}^{(\tau)}}{ 1 - \beta_1^{\tau} } \right),\\
&= \boldsymbol{\theta}^{(\tau-1)} - \frac{\eta}{\sqrt{\hat{\mb{v}}^{(\tau)}} + \epsilon} \odot \left(  \beta_1 \cdot \hat{\mb{m}}^{(\tau -1)} + \frac{ (1- \beta_1) \cdot \mb{g}^{(\tau)}}{ 1 - \beta_1^{\tau} } \right).
\end{aligned}
\end{align}
Similar to the analysis provided in Equation~(\ref{equ:NAG_Momentum_in_Nadam}), Nadam proposes to look-ahead by replacing the $\hat{\mb{m}}^{(\tau -1)}$ term used in the parentheses with $\hat{\mb{m}}^{(\tau)}$ instead, which can bring about the following updating equation:
\begin{align}
\boldsymbol{\theta}^{(\tau)} = \boldsymbol{\theta}^{(\tau-1)} - \frac{\eta}{\sqrt{\hat{\mb{v}}^{(\tau)}} + \epsilon} \odot \left(  \beta_1 \cdot \hat{\mb{m}}^{(\tau)} + \frac{ (1- \beta_1) \cdot \mb{g}^{(\tau)} }{ 1 - \beta_1^{\tau} } \right).
\end{align}

The pseudo-code of the Nadam algorithm is provided in Algorithm~\ref{alg:chapter_gradient_descent_Nadam}, where most of the code are identical to those in Algorithm~\ref{alg:chapter_gradient_descent_Adam}, except the last line in updating the variable $\boldsymbol{\theta}$.

\begin{algorithm}[H]
\small
\caption{Nadam}
\label{alg:chapter_gradient_descent_Nadam}
\begin{algorithmic}[1]
	\REQUIRE Training Set $\mc{T}$; Learning Rate $\eta$; Normal Distribution Std $\sigma$; Mini-batch Size $b$; Decay Parameters $\beta_1$, $\beta_2$.
\ENSURE  Model Parameter $\boldsymbol{\theta}$
\STATE	{Initialize parameter with Normal distribution $\boldsymbol{\theta} \sim N(0, \sigma^2)$}
\STATE	{Initialize vector $\mb{m} = \mb{0}$}
\STATE	{Initialize vector $\mb{v} = \mb{0}$}
\STATE	{Initialize step $\tau = {0}$}
\STATE	{Initialize convergence $tag = False$}
\WHILE	{$tag == False$}
\STATE	{Shuffle the training set $\mc{T}$}
\FOR	{each mini-batch $\mc{B} \subset \mc{T}$}
\STATE	{Update step $\tau=\tau+1$}
\STATE	{Compute gradient vector $\mb{g} = \nabla_{\boldsymbol{\theta}} \mc{L}({\boldsymbol{\theta}}; \mc{B})$ on the mini-batch $\mc{B}$}
\STATE	{Update vector $\mb{m} = \beta_1 \cdot \mb{m} + (1- \beta_1) \cdot \mb{g}$}
\STATE	{Update vector $\mb{v} = \beta_2 \cdot \mb{v} + (1- \beta_2) \cdot \mb{g} \odot \mb{g}$}
\STATE	{Rescale vector $\hat{\mb{m}} = \mb{m}/(1 - \beta_1^{\tau})$}
\STATE	{Rescale vector $\hat{\mb{v}} = \mb{v}/(1 - \beta_2^{\tau})$}
\STATE	{Update variable $\boldsymbol{\theta} = \boldsymbol{\theta} -  \frac{\eta}{\sqrt{\hat{\mb{v}}} + \epsilon } \odot \left(  \beta_1 \cdot \hat{\mb{m}} + \frac{ (1- \beta_1) \cdot \mb{g}}{ 1 - \beta_1^{\tau} } \right)$}
\ENDFOR
\IF		{convergence condition holds}
\STATE	{$tag = True$}
\ENDIF
\ENDWHILE
\STATE	{\textbf{Return} model variable $\boldsymbol{\theta}$}
\end{algorithmic}
\end{algorithm}

\section{Hybrid Evolutionary Gradient Descent Learning Algorithms}\label{ref:sec_Hybrid}

Adam together with its many variants have been shown to be effective for optimizing a large group of problems. However, for the non-convex objective functions of deep learning models, Adam cannot guarantee to identify the globally optimal solutions, whose iterative updating process may inevitably get stuck in local optima. The performance of Adam is not very robust, which will be greatly degraded for the objective function with non-smooth shape or learning scenarios polluted by noisy data. Furthermore, the distributed computation process of Adam requires heavy synchronization, which may hinder its adoption in large-cluster based distributed computational platforms. 

On the other hand, genetic algorithm (GA), a metaheuristic algorithm inspired by the process of natural selection in evolutionary algorithms, has also been widely used for learning the solutions of many optimization problems. In GA, a population of candidate solutions will be initialized and evolved towards better ones. Several attempts have also been made to use GA for training deep neural network models \cite{ZG18, SMCLSC17, DG17} instead of the gradient descent based methods. GA has demonstrated its outstanding performance in many learning scenarios, like {non-convex objective function containing multiple local optima}, {objective function with non-smooth shape}, as well as {a large number of parameters} and {noisy environments}. GA also fits the parallel/distributed computing setting very well, whose learning process can be easily deployed on parallel/distributed computing platforms. Meanwhile, compared with Adam, GA may take more rounds to converge in addressing optimization objective functions.

In this section, we will introduce a new optimization algorithm, namely {\our} (Genetic Adaptive Momentum Estimation) \cite{gadam}, which incorporates Adam and GA into a unified learning scheme. Instead of learning one single model solution, {\our} works with a group of potential unit model solutions. In the learning process, {\our} learns the unit models with Adam and evolves them to the new generations with genetic algorithm. In addition, Algorithm {\our} can work in both standalone and parallel/distributed modes, which will also be studied in this section.

\subsection{Gadam}\label{ref:subsec_Gadam}

\begin{figure*}[!t]
 \centering    
 \begin{minipage}[l]{1.0\columnwidth}
  \centering
    \includegraphics[width=1.0\textwidth]{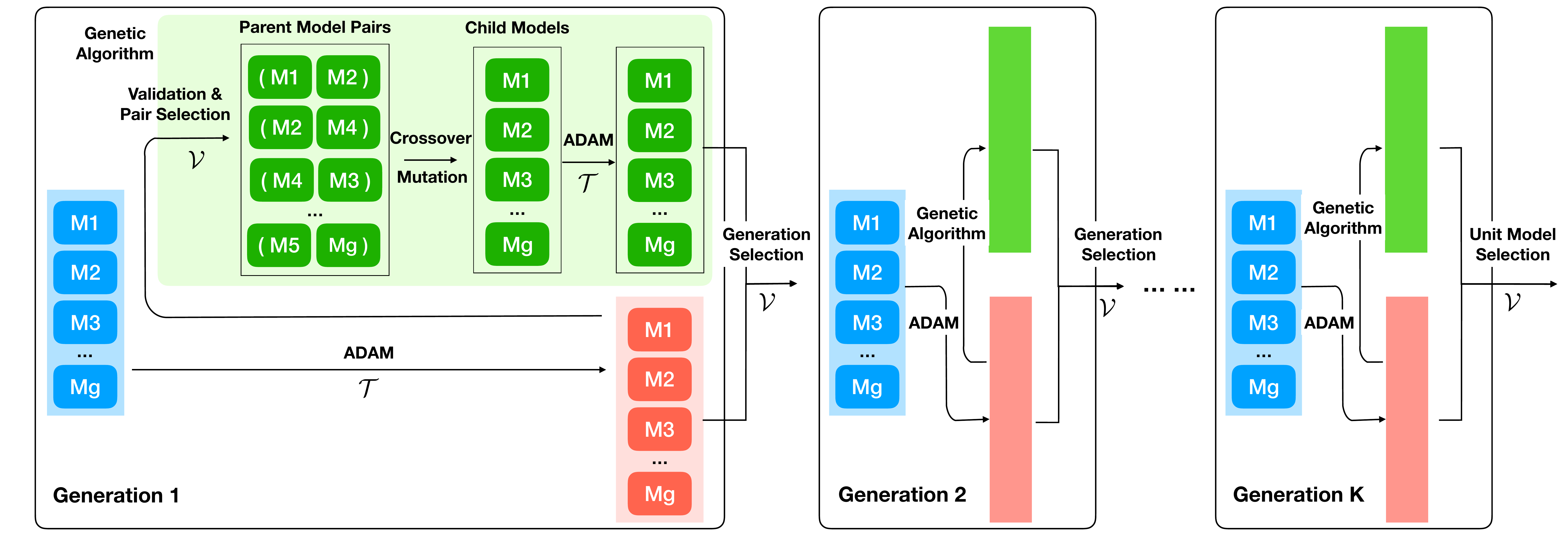}
 \end{minipage}
\caption{Overall Architecture of {\our} Model.}\label{fig:gadam_framework}
\end{figure*}

The overall framework of the model learning process in {\our} is illustrated in Figure~\ref{fig:gadam_framework}, which involves multiple learning generations. In each generation, a group of unit model variable will be learned with Adam from the data, which will also get evolved effectively via the genetic algorithm. Good candidate variables will be selected to form the next generation. Such an iterative model learning process continues until convergence, and the optimal unit model variable at the final generation will be selected as the output model solution. For simplicity, we will refer to unit models and their variables interchangeably without distinguishing their differences.

\subsubsection{Model Population Initialization}

{\our} learns the optimal model variables based on a set of unit models (i.e., variables of these unit models by default), namely the unit model population, where the initial unit model generation can be denoted as set $\mathcal{G}^{(0)} = \{M^{(0)}_1, M^{(0)}_2, \cdots, M^{(0)}_g\}$ ($g$ is the population size and the superscript represents the generation index). Based on the initial generation, {\our} will evolve the unit models to new generations, which can be represented as $\mathcal{G}^{(1)}, \mathcal{G}^{(2)}, \cdots, \mathcal{G}^{(K)}$ respectively. Here, parameter $K$ denotes the total generation number.

For each unit model in the initial generation $\mathcal{G}^{(0)}$, e.g., $M^{(0)}_i$, its variables $\boldsymbol{\theta}^{(0)}_i$ is initialized in {\our} with random values sampled from certain distributions (e.g., the standard normal distribution). These initial generation serve as the starting search points, from which {\our} will expand to other regions to identify the optimal solutions. Meanwhile, for the unit models in the following generations, their variable values will be generated via GA from their parent models respectively.

\subsubsection{Model Learning with Adam}

In the learning process, given any model generation $\mathcal{G}^{(k)}$ ($k \in \{1, 2, \cdots, K\}$), {\our} will learn the (locally) optimal variables for each unit model with Adam. Formally, {\our} trains the unit models with several epochs. In each epoch, for each unit model $M^{(k)}_i \in \mathcal{G}^{(k)}$, a separated training batch will be randomly sampled from the training dataset, which can be denoted as $\mathcal{B} = \{(\mb{x}_1, \mb{y}_1), (\mb{x}_2, \mb{y}_2), \cdots, (\mb{x}_b, \mb{y}_b)\} \subset \mathcal{T}$ (here, $b$ denotes the batch size and $\mathcal{T}$ represents the complete training set). Let the loss function introduced by unit model $M^{(k)}_i$ for training mini-batch $\mc{B}$ be $\ell(\boldsymbol{\theta}_i^{(k)})$, and the learned model variable vector by the training instance can be represented as
\begin{equation}
\bar{\boldsymbol{\theta}}_i^{(k)} = Adam\left( \ell(\boldsymbol{\theta}_i^{(k)}) \right).
\end{equation}
Depending on the specific unit models and application settings, the loss function will have different representations, e.g., mean square loss, hinge loss or cross entropy loss. Such a learning process continues until convergence, and the updated model generation $\mathcal{G}^{(k)}$ with (locally) optimal variables can be represented as $\bar{\mathcal{G}}^{(k)} = \{\bar{M}^{(k)}_1, \bar{M}^{(k)}_2, \cdots, \bar{M}^{(k)}_g\}$, whose corresponding variable vectors can be denoted as $\bar{\boldsymbol{\theta}}^{(k)}_1, \bar{\boldsymbol{\theta}}^{(k)}_2, \cdots, \bar{\boldsymbol{\theta}}^{(k)}_g\}$ respectively.

\subsubsection{Model Evolution with Genetic Algorithm}

In this part, based on the learned unit models, i.e., $\bar{\mathcal{G}}^{(k)}$, {\our} will further search for better solutions for the unit models effectively via the genetic algorithm.

\noindent \textbf{Model Fitness Evaluation}

Unit models with good performance may fit the learning task better. Instead of evolving models randomly, {\our} proposes to pick good unit models from the current generation to evolve. Based on a sampled validation batch $\mathcal{V} \subset \mathcal{T}$, the fitness score of unit models, e.g., $\bar{M}^{(k)}_i$, can be effectively computed based on its loss terms introduced on $\mathcal{V}$ as follows
\begin{equation}
\mathcal{L}^{(k)}_i = \mathcal{L}(\bar{M}^{(k)}_i; \mathcal{V}) = \sum_{(\mb{x}_j, \mb{y}_j) \in \mathcal{V}} \ell(\mb{x}_j, \mb{y}_j; \bar{\boldsymbol{\theta}}_i^{(k)}).
\end{equation}
Based on the computed loss values, the selection probability of unit model $\bar{M}^{(k)}_i$ can be defined with the following softmax equation
\begin{equation}\label{equ:probability}
P(\bar{M}^{(k)}_i) = \frac{ \exp^{(- \hat{\mathcal{L}}^{(k)}_i ) }}{ \sum_{j = 1}^g \exp^{(- \hat{\mathcal{L}}^{(k)}_j) }}.
\end{equation}
Necessary normalization of the loss values is usually required in real-world applications, as $\exp^{(- {\mathcal{L}}^{(k)}_i )}$ may approach $0$ or $\infty$ for extremely large positive or small negative loss values ${\mathcal{L}}^{(k)}_i$. As indicated in the probability equation, the normalized loss terms of all unit models can be formally represented as $[\hat{\mathcal{L}}^{(k)}_1, \hat{\mathcal{L}}^{(k)}_2, \cdots, \hat{\mathcal{L}}^{(k)}_g]^\top$, where $\hat{\mathcal{L}}^{(k)}_i \in [0, 1], \forall i \in \{1, 2, \cdots, g\}$. According to the computed probabilities, from the unit model set $\bar{\mathcal{G}}^{(k)}$, $g$ pairs of unit models will be selected with replacement as the parent models for evolution, which can be denoted as $\mathcal{P} = \{ ( \bar{M}^{(k)}_{i_1}, \bar{M}^{(k)}_{j_1} ) , ( \bar{M}^{(k)}_{i_2}, \bar{M}^{(k)}_{j_2} ) , \cdots, ( \bar{M}^{(k)}_{i_g}, \bar{M}^{(k)}_{j_g} ) \}$.

\noindent \textbf{Unit Model Crossover}

Given a unit model pair, e.g., $( \bar{M}^{(k)}_{i_p}, \bar{M}^{(k)}_{j_p} ) \in \mathcal{P}$, {\our} inherits their variables to the child model via the \textit{crossover} operation. In crossover, the parent models, i.e., $\bar{M}^{(k)}_{i_p}$ and $\bar{M}^{(k)}_{j_p}$, will also compete with each other, where the parent model with better performance tend to have more advantages. We can represent the child model generated from $( \bar{M}^{(k)}_{i_p}, \bar{M}^{(k)}_{j_p} )$ as $\tilde{M}^{(k)}_{p}$. For each entry in the weight variable $\tilde{\boldsymbol{\theta}}^{(k)}_p$ of the child model $\tilde{M}^{(k)}_{p}$, {\our} initializes its values as follows:
\begin{alignat}{2}
\tilde{\boldsymbol{\theta}}^{(k)}_p(m) = \mathbbm{1} (\mbox{rand} \le p_{i_p, j_p}^{(k)} ) \cdot \bar{\boldsymbol{\theta}}^{(k)}_{i_p}(m) + \mathbbm{1} (\mbox{rand} > p_{i_p, j_p}^{(k)} )  \cdot \bar{\boldsymbol{\theta}}^{(k)}_{j_p}(m).
\end{alignat}
In the equation, binary function $\mathbbm{1}(\cdot)$ returns $1$ iff the condition holds. Term ``$\mbox{rand}$'' denotes a random number in $[0, 1]$. The probability threshold $p_{i_p, j_p}^{(k)}$ is defined based on the parent models' performance:
\begin{equation}
p_{i_p, j_p}^{(k)} = \frac{ \exp^{ ( - \hat{\mathcal{L}}^{(k)}_{i_p} ) } }{ \exp^{ (- \hat{\mathcal{L}}^{(k)}_{i_p} ) } + \exp^{ (- \hat{\mathcal{L}}^{(k)}_{j_p}) } }.
\end{equation}
If $\hat{\mathcal{L}}^{(k)}_{i_p} > \hat{\mathcal{L}}^{(k)}_{j_p}$, i.e., model $\bar{M}_{i_p}^{(k)}$ introduces a larger loss than $\bar{M}_{j_p}^{(k)}$, we will have $0 < p_{i_p, j_p}^{(k)} < \frac{1}{2}$.

With such a process, based on the whole parent model pairs in set $\mathcal{P}$, {\our} will be able to generate the children model set as set $\tilde{\mathcal{G}}^{(k)} = \{ \tilde{M}^{(k)}_{1}, \tilde{M}^{(k)}_{2}, \cdots, \tilde{M}^{(k)}_{g} \}$.

\noindent \textbf{Unit Model Mutation}

To avoid the unit models getting stuck in local optimal points, {\our} adopts an operation called \textit{mutation} to adjust variable values of the generated children models in set $\{ \tilde{M}^{(k)}_{1}, \tilde{M}^{(k)}_{2}, \cdots, \tilde{M}^{(k)}_{g} \}$. Formally, for each child model $\tilde{M}^{(k)}_{q}$ parameterized with vector $\tilde{\boldsymbol{\theta}}^{(k)}_{q}$, {\our} will mutate the variable vector according to the following equation, where its $m_{th}$ entry can be updated as
\begin{alignat}{2}
\tilde{{\boldsymbol{\theta}}}^{(k)}_{q}(m) = \mathbbm{1}(\mbox{rand} \le p_q) \cdot \mbox{rand}(0, 1) + \mathbbm{1}(\mbox{rand} > p) \cdot \tilde{{\boldsymbol{\theta}}}^{(k)}_{q}(m),
\end{alignat}
In the equation, term $p_q$ denotes the \textit{mutation rate}, which is strongly correlated with the parent models' performance:
\begin{equation}
p_q = p \cdot \left(1 - P(\bar{M}^{(k)}_{i_q}) - P(\bar{M}^{(k)}_{j_q}) \right).
\end{equation}
For the child models with good parent models, they will have lower mutation rates. Term $p$ denotes the base mutation rate which is usually a small value, e.g., $0.01$, and probabilities $P(\bar{M}^{(k)}_{i_q})$ and $P(\bar{M}^{(k)}_{j_q})$ are defined in Equation~(\ref{equ:probability}). These unit models will be further trained with Adam until convergence, which will lead to the $k_{th}$ children model generation $\tilde{\mathcal{G}}^{(k)} = \{ \tilde{M}^{(k)}_{1}, \tilde{M}^{(k)}_{2}, \cdots, \tilde{M}^{(k)}_{g} \}$.

\subsubsection{New Generation Selection and Evolution Stop Criteria}\label{subsec:selection}

Among these learned unit models in the learned parent model set $\bar{\mathcal{G}}^{(k)} = \{\bar{M}^{(k)}_1, \bar{M}^{(k)}_2, \cdots, \bar{M}^{(k)}_g\}$ and children model set $\tilde{\mathcal{G}}^{(k)} = \{\tilde{M}^{(k)}_1, \tilde{M}^{(k)}_2, \cdots, \tilde{M}^{(k)}_g\}$, {\our} will re-evaluate their fitness scores based on a shared new validation batch. Among all the unit models in $\bar{\mathcal{G}}^{(k)} \cup \tilde{\mathcal{G}}^{(k)}$, the top $g$ unit models will be selected to form the $(k+1)_{th}$ generation, which can be formally represented as set $\mathcal{G}^{(k+1)} = \{M^{(k+1)}_1, M^{(k+1)}_2, \cdots, M^{(k+1)}_g \}$. Such an evolutionary learning process will stop if the maximum generation number has reached or there is no significant improvement between consequential generations, e.g., $\mathcal{G}^{(k)}$ and  $\mathcal{G}^{(k+1)}$: 
\begin{equation}\label{eq:convergence}
\Bigg | \sum_{M^{(k)}_i \in \mathcal{G}^{(k)}} \mathcal{L}^{(k)}_i - \sum_{M^{(k+1)}_i \in \mathcal{G}^{(k+1)}} \mathcal{L}^{(k+1)}_i \Bigg | \le \lambda,
\end{equation}
The above equation defines the stop criterion of {\our}, where $\lambda$ is the evolution stop threshold.

The optimization algorithm {\our} actually incorporate the advantages of both Adam and genetic algorithm. With Adam, the unit models can effectively achieve the (locally/globally) optimal solutions very fast with a few training epochs. Meanwhile, via genetic algorithm based on a number of unit models, it will also provide the opportunity to search for solutions from multiple starting points and jump out from the local optima. According to \cite{adam}, for a smooth function, Adam will converge as the function gradient vanishes. On the other hand, the genetic algorithm can also converge according to \cite{TGDSM94}. Based on these prior knowledge, the convergence of {\our} can be proved as introduced in \cite{gadam}. The training of unit models with {\our} can be effectively deployed on parallel/distributed computing platforms, where each unit model involved in {\our} can be learned with a separate process/server. Among the processes/servers, the communication costs are minor, which exist merely in the \textit{crossover} step. Literally, among all the $g$ unit models in each generation, the communication costs among them is $O(k \cdot g \cdot d_{\theta})$, where $d_{\theta}$ denotes the dimension of vector $\boldsymbol{\theta}$ and $k$ denotes the required training epochs to achieve convergence by Adam.

\section{A Summary}

In this paper, we have introduced the gradient descent algorithm together with its various recent variant algorithms for training the deep learning models. Based on the recent developments in this direction, this paper will be updated accordingly later in the near future.


\newpage

\vskip 0.2in
\bibliographystyle{plain}
\bibliography{reference}

\end{document}